%% file: main.tex
\documentclass[twoside]{article}
\usepackage[accepted]{aistats2018}

\usepackage{graphicx} 
\usepackage{subcaption}

\usepackage{natbib}
\usepackage{algorithm, algorithmicx, algpseudocode}
\usepackage{amsthm, amsmath, amssymb}
\usepackage{appendix}

\usepackage[nice]{nicefrac}

\newtheorem{theorem}{Theorem}

\newtheorem{corollary}{Corollary}
\newtheorem{lemma}{Lemma}

\newcommand{\expect}{\ensuremath{\mathbb{E}}}
\usepackage{color}

\begin{document} 
\runningauthor{Aditya Grover, Ramki Gummadi, Miguel L\'azaro-Gredilla, Dale Schuurmans, Stefano Ermon}
\twocolumn[
\aistatstitle{Variational Rejection Sampling}

\aistatsauthor{Aditya Grover$^\ast$ \And Ramki Gummadi$^{\ast}$ \And  Miguel L\'azaro-Gredilla}

\aistatsaddress{ Stanford University \And  Vicarious  \And Vicarious}
\aistatsauthor{Dale Schuurmans \And Stefano Ermon }

\aistatsaddress{University of Alberta \And Stanford University}
]

\begin{abstract} 
Learning latent variable models with stochastic variational inference is challenging when the approximate posterior is far from the true posterior, due to high variance in the gradient estimates. We propose a novel rejection sampling step that discards samples from the variational posterior which are assigned low likelihoods by the model. Our approach provides an arbitrarily accurate approximation of the true posterior at the expense of extra computation. Using a new gradient estimator for the resulting unnormalized proposal distribution, we achieve average improvements of $3.71$ nats and $0.21$ nats over state-of-the-art single-sample and multi-sample alternatives respectively for estimating marginal log-likelihoods using sigmoid belief networks on the MNIST dataset.
\end{abstract} 

\input{intro}

\input{prelim}

\input{framework}

\input{algorithm}
\input{expt}
\input{conclusion}

\input{ack}
\bibliography{refs}

\newpage
\newpage
\input{appendix}

\end{document}

%% file: intro.tex
\section{INTRODUCTION}
\label{sec:intro}

Latent variable models trained using stochastic variational inference can learn 
complex, high dimensional distributions~\citep{hoffman-jmlr2013,ranganath-aistats2013}. Learning typically involves maximization of a lower bound to the intractable log-likelihood of the observed data, marginalizing over the latent, unobserved variables. To scale to large datasets, inference is \textit{amortized} by introducing a recognition model approximating the true posterior over the latent variables, conditioned on the observed data~\citep{Dayan95, gershman-css2014}. The generative and recognition models are jointly trained and commonly parameterized using deep neural networks. While this provides flexibility, it also leads to expectations without any closed form expressions in the learning objective and corresponding gradients.

The general approach to stochastic optimization of such objectives involves Monte Carlo estimates of the gradients using the variational posterior (a.k.a. the recognition model) as a proposal distribution~\citep{mnih-icml2016}. A simple feed forward network, however, may not capture the full complexity of the posterior, a difficulty which shows up in practice as high variance in the gradients estimated with respect to the parameters of the proposal distribution. 

There is a vast body of prior work in variance reduction for stochastic optimization, including recent work focusing on variational methods for generative modeling. The standard approach is to use score function estimators with appropriate baselines~\citep{glynn1990likelihood,williams1992simple,fu2006gradient}. 
Many continuous distributions are also amenable to reparameterization, which transforms the original problem of taking gradients with respect to the parameters of the proposal to the simpler problem of taking gradients with respect to a deterministic function~\citep{kingma-iclr2014, rezende-icml2014, titsias2014doubly}. Finally, a complementary technique for variance reduction is the use of multi-sample objectives which compute importance weighted gradient estimates based on multiple samples from the proposal~\citep{burda-iclr2016,mnih-icml2016}. We discuss these approaches 
in Section~\ref{sec:prelim}.

In this work, we 
propose a new class of estimators for variational learning based on rejection sampling. The \textit{variational rejection sampling} approach modifies the sampling procedure into a two-step process: first, a proposal distribution (in our case, the variational posterior of a generative model) proposes a sample and then we explicitly accept or reject this sample based on a novel differentiable accept-reject test. The test is designed to reject samples from the variational posterior that are assigned low likelihoods by the generative model, wherein the threshold for rejection can be controlled based on the available computation.

We show how this procedure leads to a modification of the original variational posterior to a richer family of approximating \textit{resampled} proposal distributions. The modification is defined implicitly~\citep{mohamed2016learning} since the only requirement from the original variational posterior is that it should permit efficient sampling. Hence, our soft accept-reject test  provides a knob to smoothly interpolate between plain importance sampling with a fixed variational posterior (no rejections) to obtaining samples from the exact posterior in the limit (with potentially high rejection rate), thereby trading off statistical accuracy for computational cost. Further, even though the resampled proposal is unnormalized due to the introduction of an accept-reject test, we can surprisingly derive unbiased gradient estimates with respect to the model parameters that only require the unnormalized density estimates of the resampled proposal, 
 leading to an efficient learning algorithm.
 
Empirically, we demonstrate that variational rejection sampling outperforms competing single-sample and multi-sample approaches by $3.71$ nats and $0.21$ nats respectively on average for estimating marginal log-likelihoods using sigmoid belief networks on the MNIST dataset.

%% file: prelim.tex
\section{BACKGROUND}
\label{sec:prelim}

In this section, we present the setup for stochastic optimization of expectations of arbitrary functions with respect to parameterized distributions. We also discuss prior work applicable in the context of variational learning. We use upper-case symbols to denote probability distributions and assume they admit densities on a suitable reference measure, denoted by the corresponding lower-case notation. 

Consider the following objective:
\begin{align}\label{eq:stoc_opt_obj}
L(\theta, \phi) &= \mathbb{E}_{\mathbf{z}\sim Q_\phi}\left[f_{\theta, \phi}(\mathbf{z})\right]
\end{align}
where $\theta$ and $\phi$ denote sets of parameters and $Q_\phi$ is a parameterized sampling distribution over $\mathbf{z}$ which can be discrete or continuous. We will assume that sampling $\mathbf{z}$ from $Q_\phi$ is efficient,
and suppress subscript notation in expectations from $\mathbf{z}\sim Q_\phi$ to simply $Q$ wherever the context is clear. We are interested in optimizing the expectation of a function $f_{\theta, \phi}$ with respect to the sampling distribution $Q_\phi$ using gradient methods. In general, $f_{\theta, \phi}$ and the density $q_\phi$ need not be differentiable with respect to $\theta$ and $\phi$.

Such objectives are intractable to even evaluate in general, but unbiased estimates can be obtained efficiently using Monte Carlo techniques. The gradients of the objective with respect to $\theta$ are given by:
\begin{align*}
\nabla_\theta L(\theta, \phi) &= \mathbb{E}_{Q}  \left [ \nabla_\theta {f_{\theta, \phi}(\mathbf{z})}\right].
\end{align*}
As long as $f_{\theta, \phi}$ is differentiable with respect to $\theta$, we can compute unbiased estimates of the gradients using Monte Carlo. There are two primary class of estimators for computing gradients  with respect to $\phi$ which we discuss next.

\paragraph{Score function estimators.}
Using the fact that $\nabla_\phi q_\phi = q_\phi \nabla_\phi \log q_\phi$, the gradients with respect to $\phi$ can be expressed as:
\begin{align*}
\nabla_\phi L(\theta, \phi) &= \mathbb{E}_{Q}  \left [ \nabla_\phi {f_{\theta, \phi}(\mathbf{z})}\right] + \mathbb{E}_{Q} \left [  f_{\theta, \phi}(\mathbf{z}) \nabla_\phi {\log q_\phi(\mathbf{z})} \right]. 
\end{align*}
The first term can be efficiently estimated using Monte Carlo if $f_{\theta, \phi}$ is differentiable with respect to $\phi$. The second term, referred to as the \textit{score function estimator} or the likelihood-ratio estimator or REINFORCE by different authors~\citep{fu2006gradient,glynn1990likelihood,williams1992simple}, requires gradients with respect to the log density of the sampling distribution and can suffer from large variance~\citep{glasserman2013monte,schulman2015gradient}.
Hence, these estimators are used in conjunction with control variates (also referred to as baselines). A control variate, $c$, is any constant or random variable (could even be a function of $\mathbf{z}$ if we can correct for its bias) 
positively correlated with $f_{\theta, \phi}$ that reduces the variance of the estimator without introducing any bias: 
\begin{align*}
\mathbb{E}_{Q} \left [  f_{\theta, \phi}(\mathbf{z}) \nabla_\phi {\log q_\phi(\mathbf{z})}\right] &= \mathbb{E}_{Q} \left [ ( f_{\theta, \phi}(\mathbf{z}) -c ) \nabla_\phi {\log q_\phi(\mathbf{z})}\right]. 
\end{align*}

\paragraph{Reparameterization estimators.}
Many continuous distributions can be \textit{reparameterized} such that it is possible to obtain samples from the original distribution by applying a deterministic transformation to a sample from a fixed distribution~\citep{kingma-iclr2014,rezende-icml2014,titsias2014doubly}. For instance, if the sampling distribution is an isotropic Gaussian, $Q_\phi=\mathcal{N}(\boldsymbol{\mu}, \sigma^2 \mathbf{I})$, then a sample $\mathbf{z}\sim Q_\phi$ can be equivalently obtained by sampling $\boldsymbol{\epsilon}\sim \mathcal{N}(\mathbf{0},\mathbf{I})$ and passing 
through a deterministic function, $\mathbf{z} = g_{\boldsymbol{\mu}, \sigma}(\boldsymbol{\epsilon}) = \boldsymbol{\mu} + \sigma\boldsymbol{\epsilon}$. This allows exchanging the gradient and expectation, giving a gradient with respect to $\phi$ after reparameterization as: 
\begin{align*}
\nabla_\phi L(\theta, \phi) &= \mathbb{E}_{\boldsymbol{\epsilon} \sim S}\left[ \nabla_{\mathbf{z}} f_{\theta, \phi}(\mathbf{z}) \nabla_\phi g_{\phi}(\boldsymbol{\epsilon})\right]
\end{align*}
where $S$ is a fixed sampling distribution and $\mathbf{z} = g_{\phi}(\boldsymbol{\epsilon})$ is a deterministic transformation. Reparameterized gradient estimators typically have lower variance 
but are not widely applicable since they require $g_{\phi}$ and $f_{\theta, \phi}$ to be differentiable with respect to $\phi$ and $\mathbf{z}$ respectively unlike score function estimators~\citep{glasserman2013monte,schulman2015gradient}. Recent work has tried to bridge this gap by reparameterizing continuous relaxations to discrete distributions (called Concrete distributions) that gives low variance, but biased gradient estimates~\citep{maddison2016concrete,jang2016categorical} and deriving gradient estimators that interpolate between score function estimators and reparameterization estimators for distributions that can be simulated using acceptance-rejection algorithms, such as the Gamma and Dirichlet distributions~\citep{ruiz2016generalized,naesseth2017reparameterization}. Further reductions in the variance of reparameterization estimators is possible as explored in recent work, potentially introducing bias~\citep{roeder2017sticking,miller2017reducing,levy2017deterministic}.

\subsection{Variational learning}
We can cast variational learning as an objective of the form given in Eq.~\eqref{eq:stoc_opt_obj}. Consider a generative model that specifies a joint distribution $p_{\theta}(\mathbf{x}, \mathbf{z})$ 
over the observed variables $\mathbf{x}$ and latent variables $\mathbf{z}$ respectively, parameterized by $\theta$.
We assume the true posterior $p_{\theta}(\mathbf{z}\vert\mathbf{x})$ over the latent variables is intractable, and we introduce a variational approximation to the posterior $q_\phi(\mathbf{z} \vert \mathbf{x})$  represented by a recognition network and parameterized by $\phi$.
The parameters of the generative model and the recognition network are learned jointly~\citep{kingma-iclr2014,rezende-icml2014} by optimizing an evidence lower bound (ELBO) on the marginal log-likelihood of a datapoint $\mathbf{x}$:
\begin{align}\label{eq:elbo}
\log p_\theta(\mathbf{x})
&\geq \mathbb{E}_{Q}\left[\log \frac{p_\theta(\mathbf{x},\mathbf{z})}{q_\phi(\mathbf{z} \vert \mathbf{x})}\right]
&\triangleq ELBO(\theta, \phi).
\end{align}

Besides reparameterization estimators that were introduced in the context of variational learning, there has been considerable research in the design of control variates (CV) for variational learning using the more broadly applicable score function estimators~\citep{paisley2012variational}. In particular, \cite{wingate2013automated} and \cite{ranganath-aistats2013} use simple scalar CV, NVIL proposed input-dependent CV~\citep{mnih-icml2014}, and MuProp combines input-dependent CV with deterministic first-order Taylor approximations to the mean-field approximation of the model~\citep{gu-iclr1016}. Recently, REBAR used CV based on the Concrete distribution to give low variance, unbiased gradient updates~\citep{tucker2017rebar}, which has been subsequently generalized to a more flexible parametric version in RELAX~\citep{Grathwohletal2018relax}.

In a parallel line of work, there is an increasing effort to learn models with more expressive posteriors. Major research in this direction focuses on continuous latent variable models, for \textit{e.g.}, see \cite{gregor-icml2014,gregor2015draw,salimans2015markov,rezende2015variational,chen2016variational,song2017nice,grover2018boosted} and the references therein. Closely related to the current work is \citet{gummadi-nipsaabi2014}, which originally proposed a resampling scheme to improve the richness of the posterior approximation and derived unbiased estimates of gradients for the KL divergence from arbitrary unnormalized posterior approximations. Related work for discrete latent variable models is scarce. Hierarchical models impose a prior over the discrete latent variables to induce dependencies between the variables~\citep{ranganath2016hierarchical}, which can also be specified as an undirected model~\citep{kuleshov2017neural}. 
On the theoretical side, random projections of discrete posteriors have been shown to provide tight bounds on the quality of the variational approximation~\citep{zhu2015hybrid,grover2016variational,hsu2016tight}.

\paragraph{Multi-sample estimators.} Multi-sample objectives improve the family of distributions represented by variational posteriors by trading off computational efficiency with statistical accuracy. Learning algorithms based on these objectives do not introduce additional parameters but instead draw multiple samples from the variational posterior to reduce the variance in gradient estimates as well as tighten the ELBO. A multi-sample ELBO is given as:
\begin{align}\label{eq:multi_sample_elbo}
\log p_\theta(\mathbf{x})
&\geq \mathbb{E}_{\mathbf{z}_1, \dots, \mathbf{z}_k \sim Q_\phi}\left[\log \frac{1}{k}\sum_{i=1}^k\frac{p_\theta(\mathbf{x},\mathbf{z}_i)}{q_\phi(\mathbf{z}_i \vert \mathbf{x})}\right]
\end{align}

Biased gradient estimators using similar objectives were first used by \cite{raiko2014techniques} for structured prediction. \citet{burda-iclr2016} showed that a multi-sample ELBO is a tighter lower bound on the log-likelihood than the ELBO. Further, they derived unbiased gradient estimates for optimizing variational autoencoders trained using the objective in Eq.~\eqref{eq:multi_sample_elbo}. VIMCO generalized this to discrete latent variable models with arbitrary Monte Carlo objectives using a score function estimator with per-sample control variates~\citep{mnih-icml2016}, which serves as a point of comparison in our experiments.
Recently,~\citet{naesseth2017variational} proposed an importance weighted multi-sample objective for 
probabilistic models of dynamical systems based on sequential Monte Carlo.

%% file: framework.tex
\section{THE VRS FRAMEWORK}
\label{sec:model}

\begin{figure*}[t]
\centering
\begin{subfigure}{0.18\textwidth}
        \centering
     \includegraphics[width=\columnwidth]{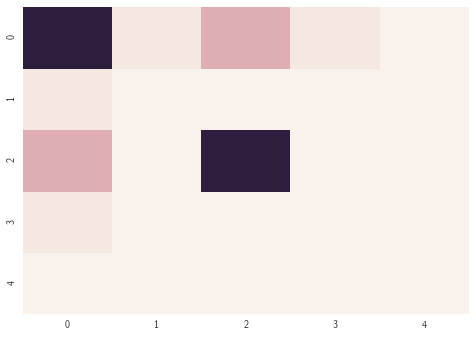} 
        \caption{Target dist.}
    \end{subfigure}
~
\begin{subfigure}{0.18\textwidth}
        \centering
        \includegraphics[width=\columnwidth]{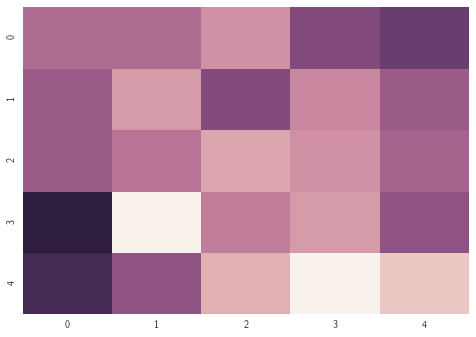} 
        \caption{$T, a, \mathtt{KL}: \infty, 1, 18$}
    \end{subfigure}
~
\begin{subfigure}{0.18\textwidth}
        \centering
        \includegraphics[width=\textwidth]{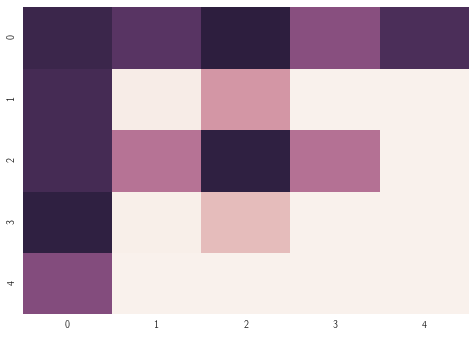} 
        \caption{$10, 0.5, 3.1$}
    \end{subfigure}
  ~
\begin{subfigure}{0.18\textwidth}
        \centering
       \includegraphics[width=\textwidth]{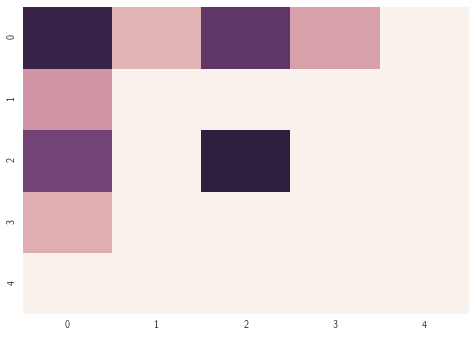} 
        \caption{$0, 0.2, 0.3$}
    \end{subfigure}
~
\begin{subfigure}{0.18\textwidth}
        \centering
        \includegraphics[width=\textwidth]{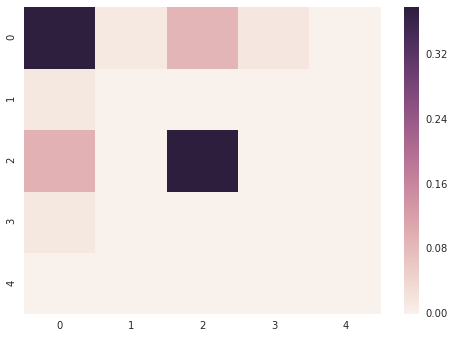} 
        \caption{$-5, 0.01, $1e-3}
    \end{subfigure}    
    \caption{The resampled posterior approximation (b-e) gets closer (in terms of $\mathtt{KL}$ divergence) to a target 2D discrete distribution (a) as we decrease the parameter $T$, which controls the acceptance probability $a$. The triples shown are $T, a, \mathtt{KL}$ divergence to target.}\label{fig:toy_1}
\end{figure*}

To motivate variational rejection sampling (VRS),  consider the ELBO objective in Eq.~\eqref{eq:elbo} for any fixed $\theta$ and $\mathbf{x}$. This is maximized when the variational posterior matches the true posterior $p_\theta(\mathbf{z} \vert \mathbf{x})$. However, in practice, the approximate posterior could be arbitrarily far from the true posterior which we seek to mitigate by rejection sampling.

\subsection{The Resampled ELBO (R-ELBO)}
Consider an alternate sampling distribution for the 
variational posterior with the density defined below:
\begin{align}\label{eq:resampled_proposal}
r_{\theta, \phi}(\mathbf{z} \vert \mathbf{x}, T) \propto q_\phi (\mathbf{z} \vert \mathbf{x}) a_{\theta, \phi}(\mathbf{z} \vert \mathbf{x}, T) 
\end{align}
where $a_{\theta, \phi}(\mathbf{z} \vert \mathbf{x}, T) \in (0, 1]$ is an acceptance probability function that could depend on $\theta, \phi$, and additional parameter(s) $T$. Unlike $p, q$, and $r$, note that $a_{\theta, \phi}(\mathbf{z} \vert \mathbf{x}, T)$ does not represent a density over the latent variables $\mathbf{z}$, but simply a function that maps each possible $\mathbf{z}$ to a number between 0-1 and hence, it denotes the probability  of acceptance for each $\mathbf{z}$.
In order to sample from $R_{\theta, \phi}$, we follow a two step sampling procedure defined in Algorithm~\ref{alg:sample}. Hence, computing Monte Carlo expectations with respect to the modified proposal involves resampling from the original proposal due to an additional accept-reject step. We refer to such sampling distributions as \textit{resampled} proposal distributions.

\begin{algorithm}
\caption{Sampler for 
$R_{\theta, \phi}(\mathbf{z} \vert \mathbf{x}, T)$}\label{alg:sample}
\begin{algorithmic}[1]
\Statex \textbf{input} $a_{\theta, \phi}(\mathbf{z} \vert \mathbf{x}, T)$, $Q_\phi(\mathbf{z}\vert\mathbf{x})$
\Statex \textbf{output} $\mathbf{z} \sim R_{\theta, \phi}(\mathbf{z} \vert \mathbf{x}, T)$
\While{True}
\State $\mathbf{z} \gets$ sample from proposal $Q_\phi(\mathbf{z}\vert\mathbf{x})$. 
\State Compute acceptance probability $a_{\theta, \phi}(\mathbf{z} \vert \mathbf{x}, T)$  
\State Sample uniform: $u \sim U[0,1]$.
\If {$u < a_{\theta, \phi}(\mathbf{z} \vert \mathbf{x}, T)$} 
\State Output sample $\mathbf{z}$.
\EndIf
\EndWhile
\end{algorithmic}
\end{algorithm}

The resampled proposal defines a new evidence lower bound on the marginal log-likelihood of
$\mathbf{x}$, which we refer to as the ``resampled ELBO", or R-ELBO:
\begin{align}\label{eq:resampled_elbo}
\log p_\theta(\mathbf{x}) &\ge   \expect_{R} \left[ \log \frac{p_\theta(\mathbf{x}, \mathbf{z}) Z_R(\mathbf{x}, T)} {q_\phi(\mathbf{z}\vert\mathbf{x})  a_{\theta, \phi}(\mathbf{z} \vert \mathbf{x}, T)} \right] \nonumber \\
&\triangleq \text{R-ELBO} (\theta, \phi).
\end{align}
where $Z_R(\mathbf{x}, T) = \mathbb{E}_{Q}[a_{\theta, \phi}(\mathbf{z} \vert \mathbf{x}, T)]$ is the (generally intractable) normalization constant for the resampled proposal distribution. To make the resampling framework described above work, we need to define a suitable acceptance function and derive low variance Monte Carlo gradient estimators with respect to $\theta $ and $\phi$ for the R-ELBO which we discuss next.

\subsection{Acceptance probability functions}

The general intuition behind designing an acceptance probability function is that it should allow for the resampled posterior to come ``close" to the target posterior $p_\theta(\mathbf{z} \vert\mathbf{x})$ (possibly at the cost of extra computation).
While there could be many possible ways of designing such acceptance probability functions, we draw inspiration from rejection sampling~\citep{halton1970retrospective}. 

In order to draw samples from a target distribution $\mathcal{T}(\mathbf{z})$, a rejection sampler first draws samples from an easy-to-sample distribution $\mathbf{z} \sim \mathcal{S}(\mathbf{z})$ with a larger-or-equal support, \textit{i.e.}, $ s(\mathbf{z}) > 0$ wherever $t(\mathbf{z}) > 0$. Then, provided we have a fixed, finite upper bound $M \in [1,\infty)$ on the likelihood ratio $\nicefrac{t(\mathbf{z})}{s(\mathbf{z})}$, we can obtain samples from the target by accepting samples from $s(\mathbf{z})$ with a probability $\frac{t(\mathbf{z})}{M s(\mathbf{z})}$. The choice of $M$ guarantees that the acceptance probability is less than or equal to 1, and overall probability of any accepted sample $\mathbf{z}$ is proportional to $\tfrac{t(\mathbf{z})}{M s(\mathbf{z})} s(\mathbf{z})$ which gives us $\mathbf{z}\sim \mathcal{T}(\mathbf{z})$ as desired. The constant $M$ has to be large enough such that the acceptance probability does not exceed $1$, but a very high value of $M$ leads to an increase in computation due to a higher rejection rate.

If the target is only known up to a normalization constant, then rejection sampling can be used provided $M$ is large enough to ensure that the acceptance probability never exceeds 1. However, we do not know in general how large $M$ should be and even if we did, it would be computationally infeasible to actually use it in a practical algorithm. A natural approximation that departs from the typical rejection sampler would be to accept proposed samples with probability  $\min\left[1, \frac{t(\mathbf{z})}{M s(\mathbf{z})}\right]$ for some $M$ that is no longer guaranteed to dominate the likelihood ratios across the entire state space. In the setting of variational learning, the target corresponds to the true, but intractable posterior that can be specified up to a normalization constant as $p_\theta(\mathbf{z} \vert \mathbf{x}) \propto p_\theta(\mathbf{x}, \mathbf{z})$ for any fixed $\theta$ and $\mathbf{x}$. 
If $Q_\phi( \mathbf{z} \vert \mathbf{x})$ denotes the proposal distribution and $M(T)$ is any function of the threshold parameter $T$,
the acceptance probability for the approximate rejection sampler is given by:
\begin{align*}
a_{\theta, \phi}(\mathbf{z} \vert \mathbf{x}, T) = \min\left[1, \frac{p_\theta(\mathbf{x}, \mathbf{z})}{M(T) q_\phi( \mathbf{z} \vert \mathbf{x})}\right].
\end{align*}
To get a fully differentiable approximation to the min operator, we consider:
\begin{align*}
a_{\theta, \phi}(\mathbf{z}\vert \mathbf{x}, T) 
&= 1/\max\left[1, \frac{M(T) q_\phi( \mathbf{z} \vert \mathbf{x})}{p_\theta(\mathbf{x}, \mathbf{z})}\right] \\
&\approx 1/\left[1^t + \left(\frac{M(T) q_\phi( \mathbf{z} \vert \mathbf{x})}{p_\theta(\mathbf{x}, \mathbf{z})}\right)^t\right]^{\nicefrac{1}{t}} 
\end{align*}
where the approximation in the last step holds for large $t$ or when any of the two terms in the $\max$ expression dominates the other. For $t=1$, we get the exponentiated negative softplus function  which we will use to be the acceptance probability function in the remainder of this paper. We leave other approximations to future work. 
Letting $T = -\log M$, the log probability of acceptance is parameterized as: 
\begin{align}\label{eq:softplus_accept_prob}
\log a_{\theta, \phi}(\mathbf{z}\vert \mathbf{x}, T) &= - \log [1+ \exp(l_{\theta, \phi}(\mathbf{z} \vert \mathbf{x}, T))] \nonumber \\
&= -\left[ l_{\theta, \phi}(\mathbf{z} \vert \mathbf{x}, T) \right]^+ 
\end{align}

where $l_{\theta, \phi}(\mathbf{z} \vert \mathbf{x}, T) = - \log p_\theta(\mathbf{x}, \mathbf{z}) + \log q_\phi(\mathbf{z} \vert \mathbf{x}) - T$ and $[*]^+$ denotes the softplus function, \textit{i.e.}, $\log(1 + e^*)$. 

Informally, the resampling scheme of Algorithm~\ref{alg:sample} with the choice of acceptance probability function as in Eq.~\eqref{eq:softplus_accept_prob} enforces the following behavior: samples from the approximate posterior that disagree (as measured by the log-likelihoods) with the target posterior beyond a level implied by the corresponding threshold $T$ have an exponentially decaying probability of getting accepted, while leaving the remaining samples with negligible interference from resampling. 

When the proposed sample $\mathbf{z}$ from $Q_{\phi}$ is assigned a small likelihood by $p_\theta$, the random variable $l_{\theta, \phi}(\mathbf{z} \vert \mathbf{x}, T)$ is correspondingly large with high probability (and linear in the negative log-likelihood assigned by $p_\theta$), resulting in a low acceptance probability. Conversely, when $p_\theta$ assigns a high likelihood to $\mathbf{z}$, we get a higher acceptance probability. Furthermore, a large value of the scalar bias $T$ results in an acceptance probability of $1$, recovering the regular variational inference setting as a special case. On the other extreme, for a small value of $T$, we get the behavior of a rejection sampler with high computational cost that is also close to the target distribution in $\mathtt{KL}$  divergence. More formally, we have Theorem~\ref{prop:monotone} which shows that the $\mathtt{KL}$ divergence can be improved monotonically by decreasing $T$. 
However, a smaller value of $T$ would require more aggressive rejections and thereby, more computation. 

\begin{theorem}\label{prop:monotone}
For fixed $\theta, \phi$, the $\mathtt{KL}$ divergence between the approximate and true posteriors, $\mathtt{KL}(R_{\theta, \phi}(\mathbf{z} \vert \mathbf{x}, T) \Vert P_\theta(\mathbf{z} \vert \mathbf{x}))$ is monotonically increasing in $T$ where $R_{\theta, \phi}(\mathbf{z} \vert \mathbf{x}, T)$ is the resampled proposal distribution with the choice of acceptance probability function in Eq.~\eqref{eq:softplus_accept_prob}. Furthermore, the behavior of the sampler in Algorithm~\ref{alg:sample}  interpolates between the following two extremes:
\begin{itemize}
\item As $T \rightarrow +\infty$,  $R_{\theta, \phi}(\mathbf{z}\vert\mathbf{x}, T)$ is equivalent to $Q_\phi(\mathbf{z}\vert\mathbf{x})$, with perfect sampling efficiency for the accept-reject step \textit{i.e}., $a_{\theta, \phi}(\mathbf{z}\vert\mathbf{x}, T)\rightarrow 1$.
\item As $T \rightarrow -\infty$, $R_{\theta, \phi}(\mathbf{z}\vert\mathbf{x}, T)$ is equivalent to $P_\theta(\mathbf{z}\vert\mathbf{x})$, with the sampling efficiency of  a plain rejection sampler \textit{i.e.}, $a_{\theta, \phi}(\mathbf{z}\vert\mathbf{x}, T)\rightarrow 0 ~\forall\mathbf{z}$.
\end{itemize} 
\end{theorem}

This phenomenon is illustrated in Figure~\ref{fig:toy_1} where we approximate an example 2D discrete target distribution on a $5 \times 5$ grid, with a uniform proposal distribution plus resampling. With no resampling ($T=\infty$), the approximation is far from the target. As $T$ is reduced, Figure~\ref{fig:toy_1} demonstrates progressive improvement in the posterior quality both visually as well as via an estimate of the $\mathtt{KL}$ divergence from approximation to the target along with an increasing computation cost reflected in the lower acceptance probabilities. In summary, we can express the R-ELBO as:
\begin{equation}
\label{eq:relbo2}
\text{R-ELBO}(\theta, \phi) =
\log p_\theta(\mathbf{x}) - \mathtt{KL}(R_{\theta, \phi}(\mathbf{z} \vert \mathbf{x}, T) \Vert P_\theta(\mathbf{z} \vert \mathbf{x})).
\end{equation}

Theorem~\ref{prop:monotone} and Eq.~\eqref{eq:relbo2} give the following corollary.

\begin{corollary}\label{cor:relbo}
The R-ELBO gets tighter by decreasing $T$ but more expensive to compute. 
\end{corollary}
With an appropriate acceptance probability function, we can therefore traverse the computational-statistical trade off for maximum likelihood estimation by adaptively tuning the threshold $T$ based on 
available computation. 

\subsection{Gradient estimation}
The resampled proposal distribution in Eq.~\eqref{eq:resampled_proposal} is unnormalized with an intractable normalization constant, $Z_R(\mathbf{x}, T) = \mathbb{E}_{Q}\left[  a_{\theta, \phi}(\mathbf{z}\vert\mathbf{x}, T)\right]$.
The presence of an intractable normalization constant seems challenging for both evaluation and stochastic optimization of the R-ELBO. 
Even though the constant cannot be computed in closed form,\footnote{Note that Monte Carlo estimates of the partition function can be obtained efficiently for evaluation.} 
we can nevertheless compute Monte Carlo estimates of its gradients, as we show in Lemma~\ref{prop:kld} in the appendix.
The resulting R-ELBO gradients are summarized below in Theorem~\ref{prop:gradients}.

\begin{theorem}\label{prop:gradients}
Let $\mathtt{COV}_R(A(\mathbf{z}), B(\mathbf{z}))$ denote the covariance of the two random variables $A(\mathbf{z})$ and $B(\mathbf{z})$, where $\mathbf{z}\sim R_{\theta, \phi}$. Then:

\begin{itemize}
\item 
The R-ELBO gradients with respect to $\theta$:
{\small
\begin{align*}
\nabla_{\phi} \text{R-ELBO}(\theta, \phi) &= \mathtt{COV}_R \left( A_{\theta, \phi}(\mathbf{z} \vert \mathbf{x}, T),  B_{\theta, \phi}(\mathbf{z} \vert \mathbf{x}, T) \right)
\end{align*}
}%
where the covariance is between the following r.v.:
{\small
\begin{align*}
A_{\theta, \phi}(\mathbf{z}\vert\mathbf{x}, T) &\triangleq \log p_\theta(\mathbf{x}, \mathbf{z}) - \log q_\phi(\mathbf{z} \vert \mathbf{x}) - [l_{\theta, \phi}(\mathbf{z} \vert \mathbf{x}, T)]^+\\
B_{\theta, \phi}(\mathbf{z}\vert\mathbf{x}, T) &\triangleq \left( 1 - \sigma(l_{\theta, \phi}(\mathbf{z} \vert \mathbf{x}, T)) \right) \nabla_{\phi} \log q_{\phi}(\mathbf{z} \vert\mathbf{x}) ).
\end{align*}
}%
\item The R-ELBO gradients with respect to $\phi$:
{\small
\begin{align*}
&\nabla_{\theta} \text{R-ELBO}(\theta, \phi) = \expect_R\left[ \nabla_{\theta} \log p_\theta(\mathbf{x}, \mathbf{z}) \right] \\
&- \mathtt{COV}_R( A_{\theta, \phi} (\mathbf{z}|\mathbf{x}, T),\sigma(l_{\theta, \phi}(\mathbf{z} \vert\mathbf{x}, T)) \nabla_{\theta} \log p_\theta(\mathbf{x}, \mathbf{z}))
\end{align*}
}%
\end{itemize}
where $\sigma(\ast)$ denotes the sigmoid function applied to $\ast$, \textit{i.e.}, $\sigma(\ast) = 1/(1+\exp(-\ast))$.
\end{theorem}
In the above expressions, the gradients are expressed as the covariance of two random variables that are a function of the latent variables sampled from the approximate posterior $R_{\theta, \phi}(\mathbf{z} \vert \mathbf{x}, T)$. Hence, we only need samples from $R_{\theta, \phi}$ for learning, which can be done using Algorithm~\ref{alg:sample} followed by Monte Carlo estimation analogous to estimation of the usual ELBO gradients.

%% file: algorithm.tex
\section{LEARNING ALGORITHM}

\begin{algorithm}[t]
\caption{Variational Rejection Sampling}\label{alg:overall}
\begin{algorithmic}[1]
\Statex \textbf{input}
Network architectures for
$p_\theta(\mathbf{x}, \mathbf{z}), q_\phi(\mathbf{z} \vert \mathbf{x})$; 
quantile hyperparameter $\gamma \in (0, 1)$; initial parameters, $\theta_0, \phi_0$; threshold update frequency,
$F$; quantile estimation sample count $N$; Covariance estimate sample count $S \ge 2$, SGD based optimizer, OPT; Dataset $\{\mathbf{x}_k\}_{k=1}^K$, Number of epochs: $N$.
\Statex \textbf{output} Final estimates, $\theta, \phi$.
\State Initialize $\theta \gets \theta_0$; $\phi \gets \phi_0$; $T(\mathbf{x}) = +\infty ~ \forall ~\mathbf{x}$.
\For{$e \in \{1, \ldots, N\}$}
	\If {$e \;\mathrm{ mod }\; F =0$ } \label{line:threshold_start} 
		\For{ each $\mathbf{x}$ in dataset}
			\State Sample $Z^N \gets \{\mathbf{z}_1,\ldots, \mathbf{z}_N\} \sim q_\phi(\mathbf{z} \vert \mathbf{x})$.
			\State $T(\mathbf{x}) \gets \hat{T}^N_\gamma(\mathbf{x}, \theta, \phi)$, the Monte Carlo estimate of Eq.~\eqref{eqn:tgamma} based on samples $Z^N$.
		\EndFor
	\EndIf\label{line:threshold_end} 
	\For{ each $\mathbf{x}$ in dataset}
    \State\label{line:grad_start} Draw $S$ independent samples $\{\mathbf{z}_1, \ldots, \mathbf{z}_S\} \sim R_{\theta, \phi}(\mathbf{z}\vert \mathbf{x}, T(\mathbf{x}))$ using Algorithm~\ref{alg:sample}.
		\State\label{line:grad_end} Use Theorem \ref{prop:gradients} and  Eq.~\eqref{eq:covest} with $\{\mathbf{z}_1, \ldots, \mathbf{z}_S\}$ to estimate gradients, $\hat{g}_\theta, \hat{g}_\phi$. 
		\State Update $\theta, \phi \gets OPT(\theta, \phi, \hat{g}_\theta, \hat{g}_\phi)$.
	\EndFor
\EndFor
\end{algorithmic}
\end{algorithm}

A practical implementation of variational rejection sampling as shown in Algorithm~\ref{alg:overall} requires several algorithmic design choices that we discuss in this section.
\subsection{Threshold selection heuristic}
The monotonicity of $\mathtt{KL}$ divergence in Theorem~\ref{prop:monotone} suggests that it is desirable to choose a value of $T$ as low as computationally feasible for obtaining the highest accuracy. 
However, the quality of the approximate posterior, $Q_\phi(\mathbf{z} \vert \mathbf{x})$ for a fixed parameter $\phi$ could vary significantly across different examples $\mathbf{x}$. This would require making $T$ dependent on $\mathbf{x}$. Although learning $T$ in a parametric way is one possibility, in this work, we restrict attention to a simple estimation based approach that reduces the design choice to a single hyperparameter that can be adjusted to trade extra computation for accuracy.
For each fixed $\mathbf{x}$, let $\mathcal{L}_{\theta, \phi}(\mathbf{x})$ denote the probability distribution of the scalar random variable, $ -\log p_\theta(\mathbf{x}, \mathbf{z}) + \log q_\phi(\mathbf{z} \vert \mathbf{x})$, where $\mathbf{z} \sim Q_\phi(\mathbf{z} \vert \mathbf{x})$. 
Let $\mathcal{Q}_\mathcal{L}$ denote the quantile function\footnote{Recall that for a given CDF $\mathcal{F}(x)$, the quantile function is its `inverse', namely $\mathcal{Q}(p) = \inf \{ x \in \mathbb{R}: p \le \mathcal{F}(x)\}$.} for any given 1-D distribution $\mathcal{L}$.
For each quantile parameter $\gamma \in (0, 1]$, we consider a heuristic family of threshold parameters given $\mathbf{x}, \phi, \theta$, defined as:
\begin{equation}\label{eqn:tgamma}
T_\gamma(\mathbf{x}, \theta, \phi) \triangleq \mathcal{Q}_{\mathcal{L}_{\theta, \phi}(\mathbf{x})}(\gamma).
\end{equation}

For example, for $\gamma=0.5$, this is the median of $\mathcal{L}_{\theta, \phi}(\mathbf{x})$. Eq.~\eqref{eqn:tgamma} implies that the acceptance probability stays roughly in the range of $\gamma$ for most samples. This is due to the fact that the negative log of the acceptance probability, defined in Eq.~\eqref{eq:softplus_accept_prob} as $\left[l_{\theta, \phi}(\mathbf{z} \vert \mathbf{x}, T) \right]^+$ is positive approximately with probability $1-\gamma$, an event which is likely to result in a rejection. 
In Algorithm~\ref{alg:overall}, we compute a Monte Carlo estimate for the threshold and denote the resulting value using $N$ samples as $\hat{T}^N_\gamma(\mathbf{x}, \theta, \phi)$ (Lines~\ref{line:threshold_start}-\ref{line:threshold_end}). This estimation is done independently from the SGD updates, once every $F$ epochs, to save computational cost, and also implies that $T$ is not continuously updated as a function of $\theta, \phi$. 
Technically speaking, this introduces a slight bias in the gradients through their dependence on $T$, but we ignore this correction since it only happens once every few epochs.

\subsection{Computing covariance estimates}
To compute an unbiased Monte Carlo estimate of the covariance terms in the gradients, we need to subtract the mean of at least one random variable while forming the product term. In order to do this in Algorithm~\ref{alg:overall} (Lines~\ref{line:grad_start}-\ref{line:grad_end}), we process a fixed batch of (accepted) samples per gradient update, and for each sample, use all-but-one to compute the mean estimate to be subtracted, similar to the local learning signals proposed in \citet{mnih-icml2016}. This requires generating $S \ge 2$ samples from $R_{\theta, \phi}$ simultaneously at each step to be able to compute each gradient. More precisely, the leave-one-out unbiased Monte Carlo estimator for the covariance of two random variables $A, B$ is defined as follows. Let $(a_1, b_1), \ldots, (a_S, b_S) \sim (A, B)$ be $S$ independent samples from the joint pair $(A, B)$, and let $\hat{m}_A$ denote the sample mean for $A$: $\hat{m}_A \triangleq \frac{1}{S} \sum_{i=1}^S a_i$. Then the covariance estimate is given by:
\begin{equation}\label{eq:covest}
\mathtt{\widehat{COV}}_R(A, B) \triangleq \frac{1}{S-1} \sum_{i=1}^S(a_i - \hat{m}_A) b_i.
\end{equation}

\subsection{Hyperparameters and overall algorithm}
In summary, Algorithm~\ref{alg:overall} involves the following hyperparameters: $S$, the number of samples for estimating covariance; $\gamma$, the quantile used for setting thresholds;  $F$, the number of epochs between updating $T(\mathbf{x})$. 

%% file: expt.tex
\section{EXPERIMENTAL EVALUATION}
\label{sec:expt}
 
\input{synthetic_experiment}

\begin{table}[t]
\centering
\caption{Test NLL (in nats) for MNIST comparing VRS with published results. Lower is better.
}
\label{tab:mnist}
\begin{subtable}[t]{\columnwidth}
\centering
\caption{Baseline results from \citet{tucker2017rebar}}
\begin{tabular}{l|cc}
Model / Architecture    & 200 & 200-200  \\ \hline
NVIL ($k=1$)   & 112.5 &  99.6              \\
MuProp & 111.7 & 99.07  \\
REBAR ($\lambda = 0.1$) & 111.7 & 99    \\
REBAR & 111.6 & 99.8    \\
Concrete ($\lambda = 0.1$) & 111.3 & 102.8  \\ \hline
VRS (IS, $\gamma = 0.95$) & 106.97 & 96.38       \\
VRS (RS, $\gamma = 0.95$) & 106.89 &  96.30   \\
VRS (IS, $\gamma = 0.9$) & 106.71 & \textbf{96.26}   \\
VRS (RS, $\gamma = 0.9$) & \textbf{106.63}&  96.36
\end{tabular}
~
\vspace{0.1in}
\end{subtable}
\begin{subtable}[t]{\columnwidth}
\centering
\caption{Baseline results from \citet{mnih-icml2016}}
\begin{tabular}{l|c}
Model / Architecture    & 200-200-200 \\ \hline
NVIL ($k=1$)  & 95.2               \\
NVIL ($k=2$)  & 93.6              \\
NVIL ($k=5$)  & 93.7              \\
NVIL ($k=10$) & 93.4              \\
NVIL ($k=50$)  & 96.2\\\hline
RWS ($k=2$)  & 94.6              \\
RWS ($k=5$)  & 93.4              \\
RWS ($k=10$) & 93.0              \\
RWS ($k=50$)  & 92.5\\\hline
VIMCO ($k=2$)  & 93.5              \\
VIMCO ($k=5$)  & 92.8              \\
VIMCO ($k=10$) & 92.6              \\
VIMCO ($k=50$)  & 91.9\\ \hline
VRS (IS, $\gamma = 0.95$)  & 92.01 \\
VRS (RS, $\gamma = 0.95$) &  91.93 \\
VRS (IS, $\gamma = 0.9$) &  92.09 \\
VRS (RS, $\gamma = 0.9$) &  \textbf{91.69} \\
\end{tabular}
\end{subtable}
\vspace{-0.1in}
\end{table}
\subsection{Generative modeling}
We trained sigmoid belief networks (SBN) on the binarized MNIST dataset~\citep{lecun2010mnist}.
Following prior work, the SBN benchmark architectures for this task consist of several linear layers of 200 hidden units, and the recognition model has the same architecture in the reverse direction. Training such models is sensitive to choices of hyperparamters, and hence we directly compare VRS with published baselines in Table~\ref{tab:mnist}. The hyperparameter details for SBNs trained with VRS are given in the Appendix.

With regards to key baseline hyperparameters in Table~\ref{tab:mnist}, Concrete and REBAR specify a temperature controlling the degree of relaxation (denoted by $\lambda$), whereas multi-sample estimators based on importance weighting specify the number of samples, $k$ to trade-off computation for statistical accuracy. The relevant parameter, $\gamma$ in our case is not directly comparable but we report results for values of $\gamma$ where empirically, the average number of rejections per training example was somewhere between drawing $k=5$ and $k=20$ samples for an equivalent importance weighted objective for both $\gamma=0.95$ and $\gamma=0.9$ (with the latter requiring more computation). Additionally, we provide two estimators for evaluating the test lower bound for VRS. For the importance sampled (IS) version, we simply evaluate the ELBO using importance sampling with the original posterior, $Q_\phi$. The resampled (RS) version, on the other hand uses the resampled proposal, $R_{\theta, \phi}$ with the partition function, $Z_R$ estimated as a Monte Carlo expectation.

From the results, we observe that VRS outperforms other methods, including multi-sample estimators with $k$ as high as $50$ that require much greater computation than the VRS models considered. Generally speaking, the RS estimates are better than the corresponding IS estimates, and decreasing $\gamma$ improves performance (at the cost of increased computation). 

%% file: synthetic_experiment.tex
We evaluated variational rejection sampling (VRS) against competing methods on a diagnostic synthetic experiment and a benchmark density estimation task.

\subsection{Diagnostic experiments}
In this experiment, we consider a synthetic setup that involves fitting an approximate posterior candidate from a constrained family to a fixed target distribution that clearly falls outside the approximating family. We restrict attention to training a 1-D parameter, $\phi$ exclusively (\textit{i.e.}, we do not consider optimization over $\theta$
), and for the non-amortized case (\textit{i.e.}, conditioning on $\mathbf{x}$ is not applicable). The target distribution is 1-D, with support on non-negative integers, $z \in \{0, 1, \ldots\}$ and denoted as $P(z)$. This distribution, visualized in Figure \ref{fig:target_poisson}, is obtained by removing the mass on the first $c$ integers of a Poisson distribution with rate $\lambda^* > 0$. More details are given in the Appendix. The approximate proposal is parameterized as $Q_\phi \triangleq \mathtt{Poi}(e^{\phi})$, where $\phi$ is an unconstrained scalar, and denotes a (unmodified) Poisson distribution with the (non-negative) rate parameter, $e^{\phi}$. Note that for $\mathtt{Poi}(e^\phi)$ to explicitly represent a small mass on $z < c$ would require $\phi \rightarrow \infty$, but this would be a bad fit for points just above $c$. As a result, $\{Q_\phi\}$ does not contain candidates close to the target distribution in the sense of $\mathtt{KL}$ divergence, even a simple resampling modification could transform the raw proposal $Q_\phi$ into a better candidate approximation, $R$.

\begin{figure}[t]
\centering
\begin{minipage}{.48\columnwidth}
\centering
\includegraphics[width=\columnwidth]{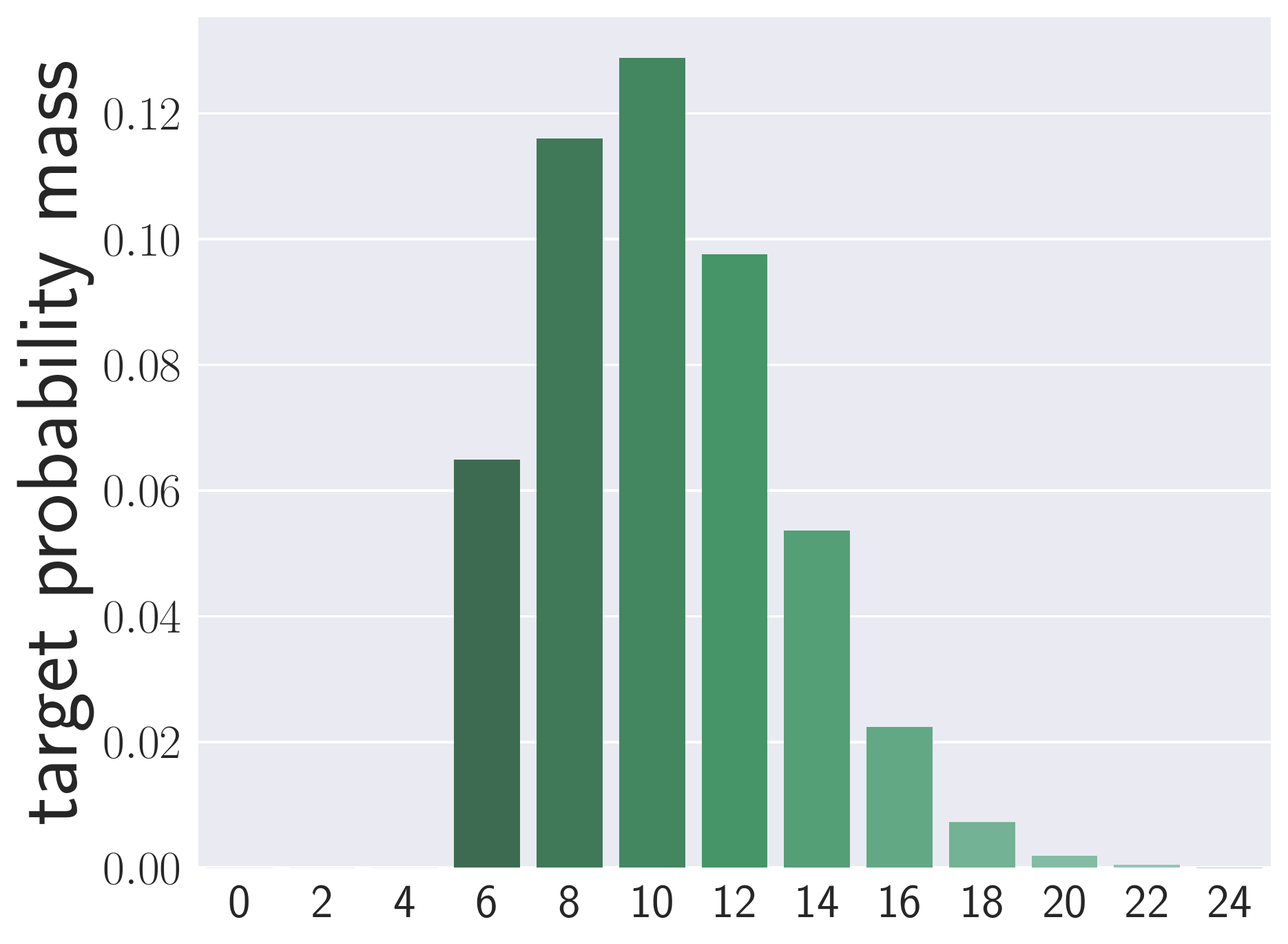}
\caption{Target distribution, $P$.
}
\label{fig:target_poisson}
\end{minipage}
\begin{minipage}{.50\columnwidth}
\centering
\includegraphics[width=\columnwidth]{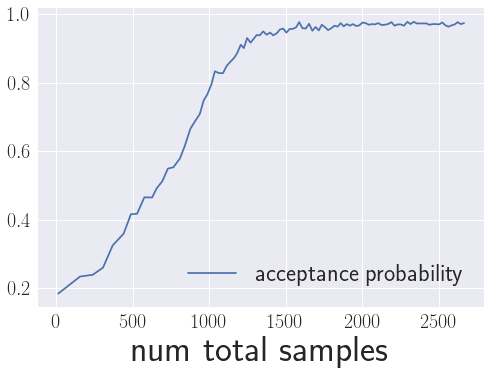} 
 \caption{Acceptance probability vs. SGD
iteration
 }
 \label{fig:accept_rbn}
\end{minipage}

\begin{subfigure}{0.48\columnwidth}
        \centering
        \includegraphics[width=\columnwidth]{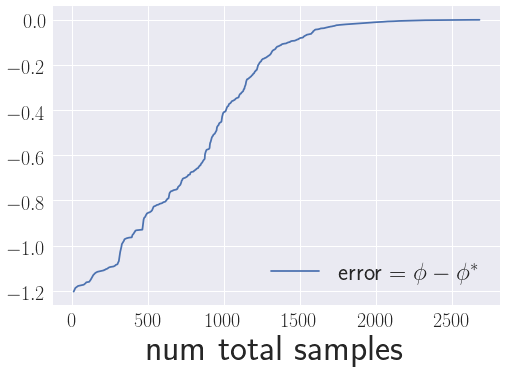} 
        \caption{Error: $\phi - \phi^*$} \label{fig:error_rbn}
    \end{subfigure}
~
\begin{subfigure}{0.48\columnwidth}
        \centering
        \includegraphics[width=\columnwidth]{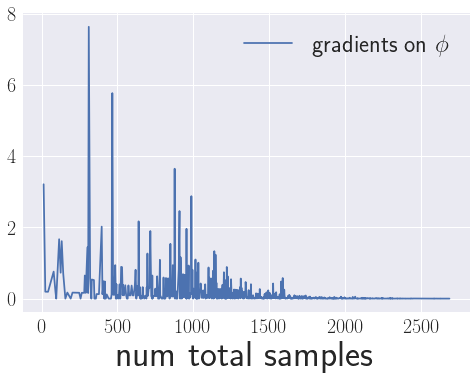} 
        \caption{Gradients for $\phi$.} \label{fig:grads_rbn}
    \end{subfigure}
\caption{VRS learning dynamics. The x-axis shows the number of total samples (both accepted and rejected) at each SGD iteration.}\label{fig:toy_rbn}

\begin{subfigure}{0.48\columnwidth}
        \centering
        \includegraphics[width=\columnwidth]{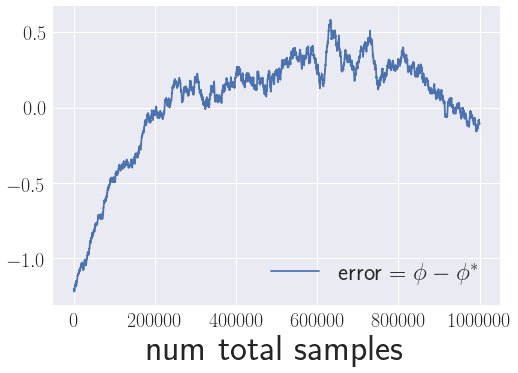} 
        \caption{Error: $\phi - \phi^*$.} \label{fig:error_vimco}
    \end{subfigure}
    ~
\begin{subfigure}{0.48\columnwidth}
        \centering
        \includegraphics[width=\columnwidth]{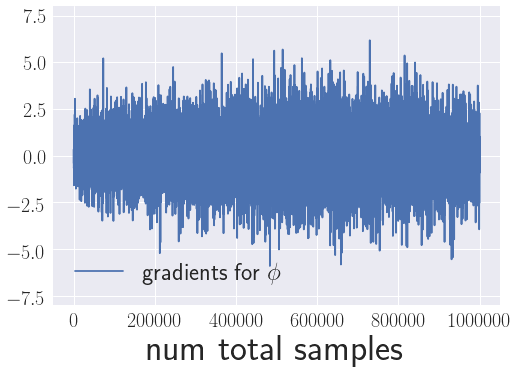} 
        \caption{Gradients for $\phi$.} \label{fig:grads_vimco}
    \end{subfigure}
    \caption{VIMCO learning dynamics. The x-axis shows the number of total samples, which is equal to $k$ times the number of iterations at each SGD iteration.}\label{fig:toy_vimco}
\vskip -0.2in
\end{figure}

In Figures~\ref{fig:accept_rbn} and \ref{fig:toy_rbn} we illustrate the dynamics of SGD using VRS gradients for approximating $P$.
To keep the analysis simple, the threshold $T$ was kept fixed at a constant value during learning. Figure~\ref{fig:accept_rbn} shows the efficiency of the sampler improving as the learning progresses due to a better fit to the target distribution. Figure~\ref{fig:error_rbn} shows the difference between the current parameter $\phi$ and $\phi^* = \log \lambda^*$ from the target distribution, quickly converging to $0$ as learning proceeds. 

As a benchmark, we evaluated the dynamics based on VIMCO gradients~\citep{mnih-icml2016}. Figure~\ref{fig:toy_vimco} suggests that the signal in gradients is too low (\textit{i.e.}, high variance in gradient estimates).
This behavior was persistent even with much smaller learning rates and large sample sizes compared to VRS gradients. One explanation is that the VIMCO gradient update for $\phi$ has a term that assigns the same average weight to the entire batch of samples, both good and bad ones (see Eq.~(8) in \citet{mnih-icml2016}). In contrast, Algorithm~\ref{alg:sample} discards rejected samples from contributing to the gradients explicitly. Yet another qualitative aspect that distinguishes VRS gradients from importance weighted multi-sample objective gradients is that Algorithm \ref{alg:sample} can dynamically adapt the amount of additional computation spent in resampling based on sample quality, as opposed to being fixed in advance.

%% file: conclusion.tex
\section{CONCLUSION}
We presented a rejection sampling framework for variational learning in generative models that is theoretically principled and allows for a flexible trade-off between computation and statistical accuracy by improving the quality of the variational approximation made by any parameterized model. We demonstrated the practical benefits of our framework over competing alternatives based on multi-sample objectives for variational learning of discrete distributions. 

In the future, we plan to generalize VRS while exploiting factorization structure in the generative model based on intermediate resampling checkpoints \citep{gummadi-nipsaabi2014}. Many baseline methods in our experiments are complementary and could benefit from VRS.
Yet another direction involves the applications to stochastic optimization problems arising in reinforcement learning and structured prediction. 

%% file: ack.tex
\section*{ACKNOWLEDGEMENTS}
We are thankful to Zhengyuan Zhou and Daniel Levy for helpful comments on early drafts. This research has been supported by a Microsoft Research PhD fellowship in machine learning for the first author, NSF grants \#1651565, \#1522054, \#1733686, Toyota Research Institute, Future of Life Institute, and Intel.

%% file: appendix.tex
\onecolumn
\section*{Appendix}
\begin{appendices}
\section{Proofs of theoretical results}

\subsection{Theorem~\ref{prop:monotone}}
\begin{proof}
We can explicitly write down the acceptance probability function as:
\begin{align*}
a_{\theta, \phi}(\mathbf{z} \vert \mathbf{x}, T) &= e^{-[l_{\theta, \phi}(\mathbf{z} \vert \mathbf{x}, T)]^+}\\
&= 
\frac{e^T p_\theta(\mathbf{x}, \mathbf{z})}{ e^T p_\theta(\mathbf{x}, \mathbf{z}) + q_\phi(\mathbf{z} \vert \mathbf{x})}
\end{align*}

From the above equation, it is easy to see that as $T \rightarrow \infty$, we get an acceptance probability close to 1, resulting in an approximate posterior close to the original proposal, $q_\phi(\mathbf{z} \vert \mathbf{x})$, whereas with $T \rightarrow -\infty$, the acceptance probability degenerates to a standard rejection sampler with acceptance probability close to $e^T \frac{p_\theta(\mathbf{x}, \mathbf{z})}{q_\phi(\mathbf{z} \vert \mathbf{x})}$, but with potentially untenable efficiency. Intermediate values of $T$ can interpolate between these two extremes. 

To prove monotonicity, we first derive the partial derivative of the $\mathtt{KL}$ divergence with respect to $T$ as a covariance of two random variables that are monotone transformations of each other. To get the derivative, we use the fact that the gradient of the $\mathtt{KL}$ divergence is the negative of the ELBO gradient derived in Theorem \ref{prop:kld}. Recall that the ELBO and the KL divergence add up to a constant independent of $T$, and that the expressions for the gradients with respect to $T$ and $\phi$ are functionally the same. We have: 
$$ \nabla_T \mathtt{KL} (R_{\theta, \phi}(\mathbf{z} \vert \mathbf{x}, T) \Vert P_\theta(\mathbf{z} \vert \mathbf{x})) =  - \mathtt{COV}_R \left(A(\mathbf{z}), \nabla_{T}  \log \gamma_r(\mathbf{z})\right),$$ where: 
\begin{align*}
A(\mathbf{z}) &= \log p_\theta(\mathbf{x}, \mathbf{z}) -  \log \gamma_r(\mathbf{z}) \\
&= \log p_\theta(\mathbf{x}, \mathbf{z}) - \log q_\phi(\mathbf{z} \vert \mathbf{x}) \\
&+ \left[ \log q_\phi(\mathbf{z}\vert\mathbf{x}) - \log p_\theta(\mathbf{x}, \mathbf{z}) - T \right]^{+} \\
&= [l_{\theta, \phi}(\mathbf{z}\vert \mathbf{x}, T)]^+ - l_{\theta, \phi}(\mathbf{z}\vert\mathbf{x}, T) - T.
\end{align*} 
For the second term in the covariance, we can use the expressions from Eq.~\eqref{eq:resampled_proposal} and Eq.~\eqref{eq:softplus_accept_prob} to write: 
\begin{align*}
\nabla_{T}  \log \gamma_r(\mathbf{z}) &=  - \nabla_{T} [l_{\theta, \phi}(\mathbf{z}\vert\mathbf{x}, T)]^+ \\
&= - \sigma(l_{\theta, \phi}(\mathbf{z}\vert\mathbf{x}, T)) \nabla_{T} l_{\theta, \phi}(\mathbf{z}\vert\mathbf{x}, T) \\
&= \sigma(l_{\theta, \phi}(\mathbf{z}\vert\mathbf{x}, T)),
\end{align*}
where $\sigma(x) \triangleq 1 / (1 + e^{-x})$ is the sigmoid function. Putting the two terms together, we have: 
\begin{align*}
&\nabla_T \mathtt{KL} (R_{\theta, \phi}(\mathbf{z} \vert \mathbf{x}, T) \Vert P_\theta(\mathbf{z} \vert \mathbf{x})) = \\& - \mathtt{COV}_R ( [l_{\theta, \phi}(\mathbf{z}\vert\mathbf{x}, T)]^+ - l_{\theta, \phi}(\mathbf{z}\vert\mathbf{x}, T) - T, ~ \sigma(l_{\theta, \phi}(\mathbf{z}\vert\mathbf{x}, T))).
\end{align*} 
To prove that the two random variables, $[l_{\theta, \phi}(\mathbf{z}\vert\mathbf{x}, T)]^+ - l_{\theta, \phi}(\mathbf{z}\vert\mathbf{x}, T) - T$ and $\sigma(l_{\theta, \phi}(\mathbf{z}\vert\mathbf{x}, T))$ are a monotone transformation of each other, we can use the identity $[x]^+ - x = \log(1 + e^x) - x = - \log \sigma(x)$ to rewrite the final expression for the gradient of the $\mathtt{KL}$ divergence as:
\begin{align*}
\nabla_T \mathtt{KL} (R_{\theta, \phi}(\mathbf{z} \vert \mathbf{x}, T) \Vert P_\theta(\mathbf{z} \vert \mathbf{x})) &= \mathtt{COV}_R \left( \log \sigma(l_{\theta, \phi}(\mathbf{z}\vert\mathbf{x}, T)) + T, ~ \sigma(l_{\theta, \phi}(\mathbf{z}\vert\mathbf{x}, T))\right)
\end{align*} 
The inequality follows from the fact that the covariance of a random variable and a monotone transformation (the logarithm in this case) is non-negative. 
\end{proof}

\subsection{Theorem~\ref{prop:gradients}}
Before proving Theorem~\ref{prop:gradients}, we first state and prove an important lemma\footnote{We assume a discrete and finite state space in all proofs below for simplicity/clarity, but when combined with the necessary technical conditions required for the existence of the corresponding integrals, they admit a straightforward replacement of sums with integrals.}.

\begin{lemma}\label{prop:kld}
Suppose $p(x) = \gamma_p(x)/Z_p$ and $r(x) = \gamma_r(x)/Z_R$ are two unnormalized densities, where only $R$ depends on $\phi$ (the recognition network parameters), but both $P$ and $R$ can depend on $\theta$.\footnote{The dependence for $R$ on $\theta$ can happen via some resampling mechanism that is allowed to, for example, evaluate $\gamma_p$ on the sample proposals before making its accept/reject decisions, as in our case.} Let $\mathcal{A}(x) \triangleq \log \gamma_p(x) - \log \gamma_r(x)$. Then the variational lower bound 
objective (on $\log Z_P$)
and its gradients with respect to the parameters $\theta, \phi$ are given by:
\begin{align*}
ELBO(\theta, \phi) &\triangleq \expect_R \left[ \mathcal{A}(x) \right] + \log Z_R\\
\nabla_{\phi} ELBO(\theta, \phi) &= \mathtt{COV}_R \left(\mathcal{A}(x), \nabla_{\phi}  \log \gamma_r(x)\right) \\
\nabla_{\theta} ELBO(\theta, \phi) &= \expect_R \left[ \nabla_{\theta} \log \gamma_p (x) \right] \\
&+ \mathtt{COV}_R \left(\mathcal{A}(x), \nabla_{\theta}  \log \gamma_r(x) \right).
\end{align*}

Note that the covariance is the expectation of the product of (at least one) mean-subtracted version of the two random variables.
Further, we can also write: 
$\mathtt{KL}(R \Vert P) =  \log \left(\expect_R \left[ e^{-\bar{\mathcal{A}}(x)} \right] \right)$, where $\bar{\mathcal{A}}(x) \triangleq \mathcal{A}(x) - \expect_R \left[ \mathcal{A}(x)\right]$ is the mean subtracted version of the learning signal, $\mathcal{A}(x)$.
\end{lemma} 

\begin{proof}
The equation for the ELBO follows from the definition. For the gradients, we can write:
$\nabla_{\phi} ELBO(\theta, \phi) = D_2 - D_1 + D_3 $, where: 
\begin{align*}
D_1 &= \nabla_{\phi} \expect_R \left[ \log \gamma_r(x) \right]\\
D_2 &= \nabla_{\phi} \expect_R \left[ \log \gamma_p(x) \right]\\
D_3 &= \nabla_{\phi} \log{Z_R}
\end{align*}
Simplifying $D_1$, $D_2$, and $D_3$, we get:
\begin{align*}
D_1 &=  \nabla_{\phi} \expect_R \left[ \log \gamma_r(x) \right]\\
&= \sum_{x} \nabla_{\phi}  \left[ r(x)  \log \gamma_r(x)\right]  &\\
&=  \sum_{x} \left( \frac{r(x)}{\gamma_r(x)}  \nabla_{\phi} \gamma_r(x) + \log \gamma_r(x) ~ \nabla_{\phi} r(x)  \right)& \\
&= \frac{1}{Z_R} \nabla_{\phi} Z_R +   \sum_{x} r(x)  \log \gamma_r(x) \nabla_{\phi} \log r(x)&\\
&= D_3 + \expect_R \left[  \log \gamma_r(x) \nabla_{\phi} \log r(x)  \right]\\
D_2 &= \nabla_{\phi} \expect_R\left[ \log \gamma_p(x) \right] \\
&=  \nabla_{\phi}
\sum_{x} r(x) \log \gamma_p(x) \\
&= \sum_{x} \log \gamma_p(x) \nabla_{\phi} r(x) \\
&= \sum_{x} \log \gamma_p(x) r(x) \nabla_{\phi} \log r(x) \\
&= \expect_R \left[  \log \gamma_p(x) \nabla_{\phi} \log r(x)  \right]
\end{align*}
 which implies:

\begin{align*}\nabla_{\phi} ELBO(\theta, \phi) &= D_2 - (D_1 - D_3)\\
& = \expect_R \left[ \left( \log \gamma_p(x) - \log \gamma_r (x) \right) \nabla_{\phi}  \log r(x) \right].
\end{align*}

Next, observe that $\expect_R \left[ \nabla_{\phi}  \log r(x) \right] = 0$, Therefore, using the fact that the expectation of the product of two random variables is the same as their covariance when at least one of the two random variables has a zero mean, we get $\nabla_{\phi} ELBO(\theta, \phi) = \mathtt{COV}_R \left(\mathcal{A}(x), \nabla_{\phi}  \log r(x)\right)$. Next note that we can add an arbitrary constant to either random variable without changing the covariance, therefore this is equal to $\mathtt{COV}_R \left(\mathcal{A}(x), \nabla_{\phi}  \log r(x) - \nabla_\phi \log Z_R \right) = \mathtt{COV}_R \left(\mathcal{A}(x), \nabla_{\phi}  \log \gamma_r(x)\right)$.

The derivation for the gradient with respect to $\theta$ is analogous, except for $D_2$, which has an additional term $\expect_R \left[ \nabla_{\theta}\log \gamma_p (x) \right]$ which did not appear in the gradient with respect to $\phi$ because of our assumption on the lack of dependence of $\log \gamma_p(x)$ on the recognition parameters $\phi$. For the identity on the KL divergence, we have:
\begin{align*}
\mathtt{KL}(R \Vert P) &= \log{Z_P} - \log{Z_R} + \expect_R \left[ \log \gamma_r(x) - \log \gamma_p(x) \right]&\\
&= \log \left( \sum_{x} \frac{\gamma_p(x)}{Z_R} \right) + \expect_R \left[ \log \gamma_r(x) - \log \gamma_p(x) \right] &\\
& = \log \left( \expect_R \left[ \frac{\gamma_p(x)}{\gamma_r(x)}\right]\right) + \expect_R \left[ \log \gamma_r(x) - \log \gamma_p(x) \right]&\\
& = \log \left(\expect_R \left[ e^{-\mathcal{A}(x)} \right]\right) + \expect_R \left[ \mathcal{A}(x) \right] &\\
& = \log \left(\expect_R \left[ e^{-\bar{\mathcal{A}}(x)} \right] \right).
\end{align*}
\end{proof}

Using the above lemma, we provide a proof for Theorem~\ref{prop:gradients} below.
\begin{proof}
We apply the result of Theorem \ref{prop:kld}, which computes the ELBO corresponding to the two unnormalized distributions on the latent variable space $\mathbf{z}$ (for fixed $\mathbf{x}, T$), with $\log \gamma_p(.) \triangleq \log p_\theta(\mathbf{z}, \mathbf{x})$ and $\log \gamma_r(.) \triangleq \log q_\phi(\mathbf{z} \vert \mathbf{x}) - \left[ l_{\theta, \phi}(\mathbf{z} \vert \mathbf{x}, T) \right]^{+}$. This gives:
$\nabla_{\phi} \text{R-ELBO}(\theta, \phi) = \mathtt{COV}_R \left(A_{\theta, \phi}(\mathbf{z}\vert\mathbf{x}, T), \nabla_{\phi}  \log \gamma_r(\mathbf{z})\right)$. We can then evaluate 
$\nabla_{\phi} \log \gamma_r(\mathbf{z}) = \left( 1 - \sigma(l_{\theta, \phi}(\mathbf{z}\vert\mathbf{x}, T)) \right) \nabla_{\phi} \log q_{\phi}(\mathbf{z} \vert \mathbf{x})$, where $\sigma()$ is the sigmoid function. Note that this is a consequence of the fact that the derivative of the softplus, $\log (1 + e^x)$, is the sigmoid function, $1 / (1 + e^{-x})$.
Similarly for the $\theta$ gradient, we get:
\begin{align*}
\nabla_{\theta} \text{R-ELBO}(\theta, \phi) &= \expect_Q \left[ \nabla_{\theta} \log p_\theta(\mathbf{x}, \mathbf{z}) \right] \\
&+ \mathtt{COV}_R \left( A_{\theta, \phi}(\mathbf{z}\vert\mathbf{x}, T) ,\nabla_{\theta}  \log \gamma_r(\mathbf{z}) \right)
\end{align*}
where:
\begin{align*}
\nabla_{\theta}  \log \gamma_r(\mathbf{z}) &= \nabla_{\theta} \left[ l_{\theta, \phi}(\mathbf{z} \vert \mathbf{x}, T) \right]^{+}\\
&= \sigma(l_{\theta, \phi}(\mathbf{z} \vert \mathbf{x}, T)) \nabla_{\theta} l_{\theta, \phi}(\mathbf{z} \vert \mathbf{x}, T) \\
&= -  \sigma(l_{\theta, \phi}(\mathbf{z} \vert \mathbf{x}, T)) \nabla_{\theta} \log p_\theta(\mathbf{x}, \mathbf{z}).
\end{align*} 
\end{proof}

\section{Experimental details}
\subsection{Synthetic}

To construct the target distribution, we transform a Poisson distribution of rate $\lambda^* >0$, denoted $\mathtt{Poi}(\lambda^*)$ by removing probability mass near $0$. More precisely, this transformation forces a negligible uniform mass, $\epsilon \approx 0$, on $0 \le z < c$. This leaves the distribution unnormalized, although this fact is not particularly relevant for subsequent discussion. 
The approximate proposal is parameterized as $Q_\phi \triangleq \mathtt{Poi}(e^{\phi})$, where $\phi$ is an unconstrained scalar, and denotes a (unmodified) Poisson distribution with the (non-negative) rate parameter, $e^{\phi}$. Note that for $\mathtt{Poi}(e^\phi)$ to explicitly represent a small mass on $z < c$ would require $\phi \rightarrow \infty$, but this would be a bad fit for points just above $c$. As a result, $\{Q_\phi\}$ does not contain candidates close to the target distribution in the sense of $\mathtt{KL}$ divergence, even while it may be possible to approximate well with a simple resampling modification that transforms the raw proposal $Q_\phi$ into a better candidate, $R$.

The target distribution was set with an optimal parameter $\phi^* = \log(10.0)$ (\textit{i.e.}, the rate parameter is 10.0), and $c=5$. The optimizer used was SGD with momentum using a mass of 0.5. We observed that the gradients were consistently positive while initialized at a parameter setting less than the true value (as shown in the plots) and similarly consistently negative when initialized to a parameter more than the true value (which we did not present in the paper) due to the consistency of the correlation between the two terms in the covariance specific to this toy example. For resampling, plots show results with learning rate set to 0.01 and $T=50$. For VIMCO, plots show results with learning rate set to 0.005 and $k=100$.

\subsection{MNIST}
We consider the standard $50,000/10,000/10,000$ train/validation/test split of the binarized MNIST dataset. For a direct comparison with prior work, both the generative and recognition networks have the same architecture of stochastic layers. No additional deterministic layers were used for SBNs trained using VRS. The batch size was $50$, the optimizer used is Adam with a learning rate of $3e-4$. We ran the algorithm for $5,000,000$ steps updating the resampling thresholds after every $100,000$ iterations based on the threshold selection heuristic corresponding to the top $\gamma$ quantile. We set $S=5$ for the unbiased covariance estimates for gradients. The lower bounds on the test set are calculated based on importance sampling with $25$ samples for IS or resampling with $25$ accepted samples for RS.
\end{appendices}